\documentclass{article}

\usepackage{microtype}
\usepackage{graphicx}
\usepackage{subfigure}
\usepackage{booktabs} % for professional tables
\usepackage{amsmath,amssymb,amsthm,amsfonts}
\newcommand\tab[1][1cm]{\hspace*{#1}}
\usepackage{multirow}
\newcommand{\bftab}{\fontseries{b}\selectfont}
\usepackage{float}

\newtheorem{theorem}{Theorem}

\newtheorem{lemma}{Lemma}
\newtheorem{definition}{Definition}

\usepackage{hyperref}

\usepackage[accepted]{icml2018}

\icmltitlerunning{Directed Exploration in PAC Model-free Reinforcement Learning}

\begin{document}

\twocolumn[
\icmltitle{Directed Exploration in PAC Model-Free Reinforcement Learning}



\icmlsetsymbol{equal}{*}

\begin{icmlauthorlist}
\icmlauthor{Min-hwan Oh}{ieor}
\icmlauthor{Garud Iyengar}{ieor}
\end{icmlauthorlist}

\icmlaffiliation{ieor}{Columbia University, New York, NY, USA}

\icmlcorrespondingauthor{Min-hwan Oh}{m.oh@columbia.edu}

\icmlkeywords{Machine Learning, ICML}

\vskip 0.3in
]



\printAffiliationsAndNotice{}  % leave blank if no need to mention equal contribution
%\printAffiliationsAndNotice{} % otherwise use the standard text.

\begin{abstract}
We study an exploration method for model-free RL that generalizes the counter-based exploration bonus methods and takes into account long term exploratory value of actions rather than a single step look-ahead. We propose a model-free RL method that modifies Delayed Q-learning  and utilizes the long-term exploration bonus with provable efficiency. We show that our proposed method finds a near-optimal policy in polynomial time (PAC-MDP), and also provide experimental evidence that our proposed algorithm is an efficient exploration method. 
\end{abstract}

\section{Introduction}
In reinforcement learning (RL), an agent, whose objective is to maximize the expected sum of reward, initially starts to make decisions in an unknown environment. It faces trials and errors while collecting reward and information. However, it is not feasible for the agent to act near-optimally until it has explored the environment sufficiently and identified all of the opportunities for high reward. One of the fundamental challenges in RL is to balance exploration and exploitation --- whether to act not greedily action according to current estimates in order to gain new information or to act consistently with past experience to maximize reward. 

 Common dithering strategies, such as $\epsilon$-greedy, or sampling from a Boltzmann distribution (Softmax) over the learned Q-values have been widely applied to standard RL methods as exploration strategies. However these naive approaches can lead to highly inefficient exploration, in the sense that they waste exploration resources on actions and trajectories which are already well known. In other words, they are not directed towards gaining more knowledge, not biasing actions in the direction of unexplored trajectories \cite{thrun1992cient, little2013learning, osband2016generalization}.

In order to avoid wasteful exploration and guide toward more directed exploration, many of the previous work adopted exploration bonus. The most commonly used exploration bonus is based on counting. That is, for each pair $(s, a)$, maintain a integer value $n_t(s, a)$ that indicates how many times the agent performed action $a$ at state $s$ so far at time $t$.  Counter-based methods have been widely used both in practice and in theory \cite{strehl2005theoretical,strehl2008analysis,kolter2009near,bellemare2016unifying,tang2017exploration,ostrovski2017count}. However,  the limitation of these methods still exist in that the exploratory value of a state-action pair is evaluated with respect only to its immediate outcome, one step ahead \cite{choshen2018dora}.
Recent work \cite{choshen2018dora} proposes an exploration method for model-free RL that generalizes the counter-based exploration bonus methods and takes into account long term exploratory value of actions rather than a single step look-ahead. Inspired by their use of propagated exploration value, we propose a model-free RL method that utilizes this long-term exploration bonus with provable efficiency. We show that our proposed method finds a near-optimal policy in polynomial time, and give experimental evidence that it is an efficient exploration method.

\section{Preliminaries}

\subsection{Markov decision processes}
A standard assumption of RL is that the environment is (discounted-reward and finite) Markov decision processes (MDP). Here we only introduce the notational framework used in this work. A finite MDP $M$ is a tuple $(\mathcal{S}, \mathcal{A}, T, R, \gamma)$, where $\mathcal{S}$ is a finite set of states; $\mathcal{A}$ is a finite set of possible actions; $T: \mathcal{S} \times \mathcal{A} \rightarrow \mathcal{P}_S $ is the transition distribution; $R: \mathcal{S} \times \mathcal{A} \rightarrow \mathcal{P}_R$ is the reward distribution; $\gamma$ is a discount factor with $\gamma \in [0,1)$.
We assume that all the (random) immediate rewards are nonnegative and are upper-bounded by a constant $R_\text{max} \geq 0$.
A policy $\pi$ is a mapping that assigns to every history $h$ a probability mass function over the actions $\mathcal{A}$. Following a policy in the MDP means that $a_t \sim \pi(\cdot | h_t)$. A stationary policy is a mapping $\pi: \mathcal{S} \times \mathcal{A} \rightarrow [0,1]$.  The discounted state- and action-value functions will be denoted by $V^\pi$ and $Q^\pi$, respectively, for any (not necessarily stationary) policy $\pi$. The optimal state- and action-value functions will be denoted by  $V^*$ and $Q^*$.

Given $s_0, a_0, r_0, s_1, a_1, r_1, ...$ a stream of experience generated when algorithm $\mathfrak{A}$ interacts with $M$, we define
its value at time $t$, conditioned on the past $h_t$ as $V^\mathfrak{A_t}_{M} = \mathbb{E}[\sum^\infty_{k=0} \gamma^k r_{t+k}|h_t]$. Let $V_\textnormal{max}$ be an upper bound on the state values in the MDP.

\subsection{Sample Complexity}
One of the common evaluation criteria for RL algorithms is to count the number of non-optimal actions taken. Roughly, this quantity tells how many mistakes the agent make at most.

\begin{definition}[\citealt{kakade2003sample}] Let $\epsilon > 0$ be a prescribed accuracy and $\delta > 0$ be an allowed probability of
failure. The expression $\zeta(\epsilon, \delta, S, A, \gamma, R_\textnormal{max})$ is a sample complexity bound for algorithm $\mathfrak{A}$, if the following holds: Take any $\epsilon > 0$, $\delta \in (0,1)$, $S>0$, $A>0$, $\gamma \in [0,1)$, $R_\textnormal{max}>0$ and any MDP $M$ with $S$ states, $A$ actions, discount factor $\gamma$, and rewards bounded by $R_\textnormal{max}$. Let $\mathfrak{A}$ interact with $M$, resulting in the process $s_0, a_0, r_0, s_1, a_1, r_1, ...$ Then, independently of the choice of $s_0$, with probability at least $1-\delta$, the number of timesteps such that $V^\mathfrak{A_t}_{M} < V^*_M(s_t) - \epsilon$ is at most $\zeta(\epsilon, \delta, S, A, \gamma, R_\textnormal{max})$.
\end{definition}

An algorithm with sample complexity that is polynomial in $1/\epsilon, \log(1/\delta), S, A, 1/(1-\gamma), R_\textnormal{max}$ is called PAC-MDP (probably approximately correct in MDPs)

\subsection{Previous sample complexity results}
$E^3$ algorithm \citep{kearns2002near} and its successor, R-max \citep{brafman2002r}, were the first algorithms that have polynomial time bounds for finding near-optimal policies. These methods maintain a complete, but possibly inaccurate model of its environment and acts based on the optimal policy derived from this model. The model is initialized in an optimistic fashion: all actions in all states return the maximal possible reward and the model is updated each time when a state becomes known. R-max has the sample complexity of $\widetilde{O}\left( \frac{S^2 A }{\epsilon^3(1-\gamma)^6} \right)$. The MBIE algorithm \citep{strehl2005theoretical, strehl2008analysis} applies confidence bounds to compute an optimistic policy and has the same sample complexity $\widetilde{O}\left( \frac{S^2 A }{\epsilon^3(1-\gamma)^6} \right)$. There are variants of R-max algorithms, such as the OIM algorithm \citep{szita2008many} and MoRMax \citep{szita2010model}. MoRMax is shown to have the smallest sample-complexity $\widetilde{O}\left( \frac{S A }{\epsilon^2(1-\gamma)^6} \right)$ among discounted finite MDPs. All of these algorithms mentioned are model-based. Unlike aforementioned methods which build an approximate model of the environment, Delayed Q-learning \citep{strehl2006pac} rather approximate an action value function directly. Delayed Q-learning
is the first model-free method with known complexity bounds with $\widetilde{O}\left( \frac{S A }{\epsilon^4(1-\gamma)^8} \right)$ sample-complexity.

%R-max collects statistics about transitions and rewards. When visits to a state enable high precision estimations of real transition probabilities and rewards then state is declared known. R-max also maintains an approximate model of the environment. Initially, the model assumes that all actions in all states lead to a (hypothetical) maximum-reward absorbing state. The model is updated each time when a state becomes known. The optimal policy of the model is either the near-optimal policy in the real environment or enters a not-yet-known state and collects new information. R-max has the sample complexity of $\widetilde{O}\left( \frac{S^2 A }{\epsilon^3(1-\gamma)^6} \right)$. The MBIE algorithm \citep{strehl2005theoretical, strehl2008analysis} applies confidence bounds to compute an optimistic policy and has the same sample complexity $\widetilde{O}\left( \frac{S^2 A }{\epsilon^3(1-\gamma)^6} \right)$. The OIM algorithm \citep{szita2008many} separates exploration value function and operates similarly to R-max , all of which are model-based. The state-of-the-art complexity bound for model-based learning in MDPs is due to \cite{li2009unifying}, who showed that R-max takes at most $\widetilde{O}\left( \frac{S^2 A V_\textnormal{max}^3}{\epsilon^3(1-\gamma)^3} \right)$ exploratory actions. A widely known model-free method with known complexity bounds is delayed Q-learning \citep{strehl2009reinforcement}, which requires at most $\widetilde{O}\left( \frac{SAV_\textnormal{max}^4}{\epsilon^4(1-\gamma)^4} \right)$ exploratory actions.

\subsection{$E$-values}
Choshen et al. \yrcite{choshen2018dora} propose a method using a parallel $E$-value MDP which has the same transition model as the original MDP, but has no rewards associated with any of the state-actions. Hence, the true value of all state-action pairs is 0. With the initial value of 1 for all state-action pairs, they show (empirically) that these $E$-values represent the missing knowledge and thus can be used for propagating directed exploration.  Intuitively, the value of $E(s, a)$ at a given timestep during training stands for uncertainty and decreases each time the agents experiences $(s, a)$ pair. On-policy SARSA \citep{singh2000convergence} update rule is applied to the $E$-value MDP, where the acting policy is selected on the original MDP.\vspace{-0.1cm}
\begin{equation*}\vspace{-0.1cm}
E(s_t,a_t) \leftarrow (1-\alpha) E(s_t,a_t) + \alpha \gamma_E E(s_{t+1}, a_{t+1})
\end{equation*}
While the $E$-value MDP is training, the proposed method uses a log transformation applied to $E$-values to get the corresponding exploration bonus term for the original MDP. This bonus term is shown to be equivalent counter-based methods for finite MDPs when the discount factor $\gamma_E$ of the $E$-value MDP is set to 0 if a fixed learning rate $\alpha$ is used for all updates. Hence, Choshen et al. \yrcite{choshen2018dora} argue that, with $\gamma_E > 0$, the logarithm of $E$-Values can be thought of as a generalization of visit counters, with propagation of the values along state-action pairs. Although the empirical results demonstrate efficient exploration in the experiments used in their work, the theoretical analysis of their proposed algorithm is lacking, essentially only showing convergence with infinite visiting. In this work, we show that $E$-value can be incorporated in PAC-MDP with theoretical guarantee.

%Recent work by \cite{choshen2018dora} proposes an exploration method for model-free RL that generalizes the counter-based exploration bonus methods and takes into account long term exploratory value of actions rather than a single step look-ahead. 

\subsection{Delayed Q-learning}
In Delayed Q-learning \citep{strehl2006pac,strehl2009reinforcement}, the agent only observes one sample transition for the action it takes in the current state. Delayed Q-learning uses optimistic initialization of the value function, and waits until $m$ transitions from $(s, a)$ are gathered before considering an update of $Q(s, a)$ (this is where ``delay'' comes from). When $m$ is sufficiently large (but is still bounded by a polynomial), the new value of $Q(s, a)$ is still optimistic with high probability \cite{li2009unifying}. It maintains the known state-action set $K_t$, similar to the approaches introduced in earlier model-based PAC-MDP algorithms \citep{brafman2002r} as well as Boolean $LEARN$ flags for each state-action pair that is set as TRUE when a pair does not belong to the set $K_t$, which allows an update to $Q(s,a)$. These tools allow us to bound the number of occurrence of the undesired ``escape'' events from $K_t$. A variant of Delayed Q-learning uses techniques such as interval estimation to attempt an update before $m$-th time as long as the current estimate satisfies the update criterion \citep{strehl2007probably}.

\section{Directed Delayed Q-learning}

Our proposed algorithm, Directed Delayed Q-learning, maintains Q-value estimates, $Q(s,a)$ and $E$-value estimates, $E(s,a)$ for each state-action pair $(s,a)$. At each timestep $t$, let $Q_t(s,a)$ denote the algorithm’s current Q-value estimate and $E_t(s,a)$ denote its current $E$-value estimate. The agent always acts greedily with respect to Q-value estimates plus the exploration bonus, meaning that if state $s$ is the $t$-th state reached, the next action is chosen by 
\begin{equation}
a:= \arg\max_{a\in \mathcal{A}} Q_t(s,a) + \frac{\lambda}{\sqrt{\log_{\eta} E_t(s,a)}}.
\end{equation}
Let $Q'_t(s,a)$ denote $Q_t(s,a) + \lambda/\sqrt{\log_{\eta} E_t(s,a)}$ for convenience. Our proposed method is based on Delayed Q-learning \citep{strehl2006pac, strehl2009reinforcement}. Our proposed method modifies Delayed Q-learning in that we introduce an exploration bonus using $E$-values to take into account the long term exploratory value of actions and we perform delayed updates to $E$-values along with Q-values. We also adopt the interval estimation technique \citep{strehl2007probably} to update the value function whenever a current Monte Carlo estimate differs from the target value function sufficiently, instead of waiting until the agent collects a fixed number of $m$ samples to estimate a new value function for each attempted update. The term $\rho/\sqrt{n(s,a)}$ is introduced to account for Monte Carlo estimate errors in the case of a premature delay, where $n(s,a)$ is the inner counter of state-action pairs within each update and resets after a successful update or the $m$-th attempted update. Note that $n(s,a)$ differs from a global counter which keeps track of the number of state-action visits for the entire duration of learning and which is generalized by $E$-values. It is also important to note that the proposed $E$-value based exploration bonus can still be applied to the fully delayed version of Delayed Q-learning with fixed delay intervals (in fact, with the same PAC bound). We apply the interval estimation technique for empirical performance gains.  %We utilize the interval estimation technique for e 

Furthermore, there are differences between our proposed method and Choshen et al. \yrcite{choshen2018dora} in that Choshen et al. \yrcite{choshen2018dora} still apply dithering strategies ($\epsilon$-greedy and softmax policies) over the sum of Q-value and a exploration bonus based $E$-value. On the other hand, our proposed algorithm acts greedily with respect to Equation (1).
We also update $E$-values with off-policy updates rather than on-policy to ensure monotonic decrease in $E$ value for every update. 

In addition to Q-value and $E$-value estimates, similarly to Delayed Q-learning, our algorithm maintains a Boolean variable $LEARN(s, a)$, for each $(s, a)$.\footnote{The maintenance of the $LEARN(s, a)$ variable is essentially the same as Delayed Q-learning. For details, see \cite{strehl2006pac, strehl2009reinforcement}.} This variable indicates whether the agent currently considers a modification to its Q-value and $E$-value estimates. The algorithm also relies on other free parameters, $\epsilon_1 \in (0, 1)$ and a positive integer $m$, a positive real number $\lambda$, $E$-value discount factor $\gamma_E$, and the base of log transformation $\eta$. In the analysis which is provided in Appendix, we provide precise values for these parameters in terms of the other inputs $(S, A, \epsilon, \delta, \gamma)$ that guarantee the resulting algorithm is PAC-MDP. We provide an efficient implementation, Algorithm 1, of Directed Delayed Q-learning.

\begin{algorithm}[H]
   \caption{Directed Delayed Q-learning}
   \label{algo1}
\begin{algorithmic}
	\STATE {\bfseries Input:} $\gamma, S, A, \epsilon_1, m, \lambda, \gamma_E, \eta$
	\FOR{\textbf{all} $(s,a)$}
		\STATE $Q(s,a) \leftarrow 1/(1-\gamma)$ \tab// Q-value estimate
        \STATE $E(s,a) \leftarrow 1-\epsilon_1$ \tab// exploration value estimate
        \STATE $\Tilde{Q}(s,a) \leftarrow 0$ \tab// inner loop estimate for Q-values
        \STATE $\Tilde{E}(s,a) \leftarrow 0$ \tab// inner loop estimate for $E$-values
        \STATE $n(s,a) \leftarrow 0$ \tab// inner counter
        \STATE $b(s, a) \leftarrow 0$ \tab// beginning time of attempted update
        \STATE $LEARN(s, a) \leftarrow true$ \tab// the LEARN flags
	\ENDFOR
    \STATE $t^* \leftarrow 0$ \tab// time of most recent Q-value change
	\FOR{$t = 1, 2, 3,...$}
    	\STATE Let $s$ denote the state at time $t$
    	\STATE \resizebox{.9\hsize}{!}{Choose action $a := \arg\max_{a' \in \mathcal{A}} Q(s,a') + \frac{\lambda}{\sqrt{\log_{\eta} E(s,a')}}$}
        \STATE Observe immediate reward $r$ and next state $s'$ 
		\IF{$b(s,a) \leq t^*$}
			\STATE $LEARN(s,a) \leftarrow true$
		\ENDIF
		\IF{$LEARN(s, a) = true$}
			\STATE $n(s,a) \leftarrow n(s,a) + 1$
            \STATE $\alpha \leftarrow \frac{1}{n(s,a)}$
        	\STATE \resizebox{.9\hsize}{!}{$\Tilde{Q}(s,a) \leftarrow (1-\alpha)\Tilde{Q}(s,a) + \alpha \left(r + \gamma \max_{a'} Q(s',a')\right)$}
        	\STATE \resizebox{.85\hsize}{!}{$\Tilde{E}(s,a) \leftarrow (1-\alpha)\Tilde{E}(s,a) + \alpha \left( \gamma_E \max_{a'} E(s',a')\right)$}
			\IF{\resizebox{.9\hsize}{!}{$Q'(s,a)  - \left(\Tilde{Q}(s,a) + \frac{\rho}{\sqrt{n(s,a)}} + \frac{\lambda}{\sqrt{\log_{\eta} \Tilde{E}(s,a)}} \right) \geq \epsilon_1$}}
				\STATE $Q(s,a) \leftarrow \Tilde{Q}(s,a) + \frac{\rho}{\sqrt{n(s,a)}}$
                \STATE $E(s,a) \leftarrow \Tilde{E}(s,a)$
                \STATE $t^* \leftarrow t$
                \STATE \resizebox{.85\hsize}{!}{$n(s,a) \leftarrow 0; \Tilde{Q}(s,a) \leftarrow 0; \Tilde{E}(s,a) \leftarrow 0; b(s,a) \leftarrow t$}
            \ELSIF{$n(s,a) = m$}
            	\STATE \resizebox{.85\hsize}{!}{$n(s,a) \leftarrow 0; \Tilde{Q}(s,a) \leftarrow 0;, \Tilde{E}(s,a) \leftarrow 0; b(s,a) \leftarrow t$}
				\IF{$b(s,a) > t^*$}
                	\STATE $LEARN(s, a) \leftarrow false$
                \ENDIF 
			\ENDIF

		\ENDIF

    \ENDFOR
   
\end{algorithmic}
\end{algorithm}

\subsection{Update Criteria}
While the agent considers learning for a given state-action pair $(s,a)$, each time $(s,a)$ is experienced, the agents updates its surrogate Q-value and $E$-value estimates, $\Tilde{Q}$ and $\Tilde{E}$ and attempts an update to the global Q-value and $E$-value up to $m$. If the update fails even at $m$-th time, the agent discards the current surrogate estimates and starts collecting new samples. If successful, the following updates occur: $Q(s,a) \leftarrow \Tilde{Q}(s,a) + \rho/\sqrt{n(s,a)}$ and $E(s,a) \leftarrow \Tilde{E}(s,a)$. To ensure that every successful update decreases $Q'(s,a)$ by at least $\epsilon$, we require the following condition to be satisfied for an update to occur:
\resizebox{\hsize}{!}{$Q'(s,a)  - \left(\Tilde{Q}(s,a) + \frac{\rho}{\sqrt{n(s,a)}} + \frac{\lambda}{\sqrt{\log_{\eta} \Tilde{E}(s,a)}} \right) \geq \epsilon_1$}

If the above condition does not hold, then there is no update
to be performed for $Q(s,a)$ and $E(s,a)$. 

\begin{figure}[ht]
\begin{center}
\centerline{\includegraphics[width=\columnwidth]{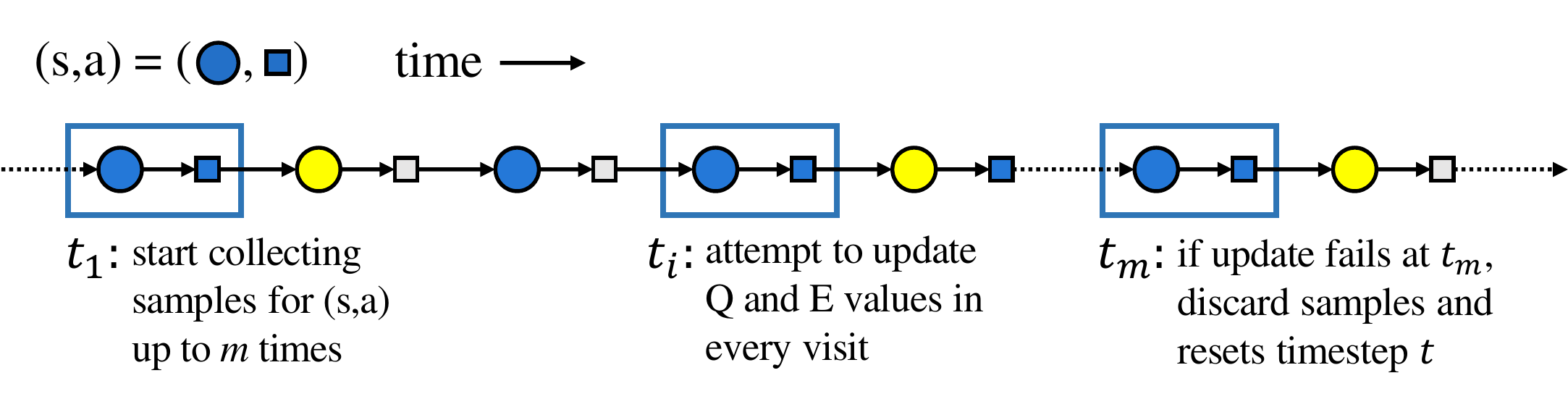}}
\caption{Flow of Q-value and $E$-value updates during execution of Directed Delayed Q-learning. Delay can be up to $m$ times per each attempted update, which may or may not succeed}
\label{icml-historical}
\end{center}
\vskip -0.1in
\end{figure}

\subsection{Main Theoretical Result}
The main theoretical result, whose proof is provided in Appendix in the supplementary material, is that the Directed Delayed Q-learning algorithm is PAC-MDP:
\vspace*{0.2cm}
\begin{theorem}
Let $M$ be any MDP and let $\epsilon$ and $\delta$ be two positive real numbers. If Directed Delayed Q-learning is executed on MDP $M$, then then
the following holds.
Let $\mathfrak{A}_t$ denote the policy of Directed Delayed Q-learning at time $t$ and $s_t$ denote the state at time $t$. With probability at least $1-\delta$, $V^\mathfrak{A_t}_{M}(s_t) \geq V^*_M(s_t) - \epsilon$ is true for all but
\begin{equation*}
O\left( \frac{SA}{\epsilon^4(1-\gamma)^8}\ln\frac{1}{\delta}\ln\frac{1}{\epsilon(1-\gamma)}\ln\frac{SA}{\delta\epsilon(1-\gamma)} \right)
\end{equation*}
timesteps.
\end{theorem}

\section{Experiments}

To assess the empirical performances of Directed Delayed Q-learning, we compared its performance to other model-free RL methods as well as different values of $\gamma_E$. Experiments were run on chain MDPs with varying length $N$. The agent begins at the far left state and at every time step has the choice to move left or right. Each move can fail with probability 0.2, which results in the opposite action. The agent receives a small reward $(r=\frac{1}{1000})$ for reaching the leftmost state, but the optimal policy is to attempt to move to the far right state and receive a much larger reward $(r=1)$. Chains with length $N = 10$ and $N = 50$ are reported below. These environments are intended to be expository rather than entirely realistic. Balancing a well known and mildly successful strategy versus an unknown, but potentially more rewarding, approach can emerge in many practical applications \citep{osband2016generalization}.

\begin{table}[H]
  \caption{Results on Chain MDPs with $N = 10$}
  \vspace{.2cm}
  \label{chainMDP-experiment_1}
\begin{center}
\begin{small}
\begin{sc}
\begin{tabular}{llr}
    \toprule
    Method   &  & Cumulative reward\\
    \midrule
    \multirow{5}{*}{\shortstack{Directed\\Delayed QL}} &$\gamma_E = 0.99$ & \bftab7089.59$\pm$48.98 \\
        &$\gamma_E = 0.90$ & 6961.50$\pm$63.63  \\
        &$\gamma_E = 0.75$ & 4530.78$\pm$94.03 \\
        &$\gamma_E = 0.50$ & 2746.06$\pm$61.84  \\
        &$\gamma_E = 0.25$ & 2624.71$\pm$14.97 \\
    Delayed QL  &  & 4325.38$\pm$59.31    \\
    QL + $\epsilon$-greedy    &  & 2435.11$\pm$134.3   \\
    \bottomrule
  \end{tabular}
\end{sc}
\end{small}
\end{center}
\vskip -0.1in
\end{table}

\begin{table}[H]
  \caption{Results on Chain MDPs with $N = 50$}
  \vspace{.2cm}
  \label{chainMDP-experiment_2}
\begin{center}
\begin{small}
\begin{sc}
\begin{tabular}{llr}
    \toprule
    Method   &  & Cumulative reward\\
    \midrule
    \multirow{5}{*}{\shortstack{Directed\\Delayed QL}} &$\gamma_E = 0.99$ & \bftab5581.02$\pm$94.72 \\
        &$\gamma_E = 0.90$ & 4982.09$\pm$116.2  \\
        &$\gamma_E = 0.75$ & 2976.96$\pm$282.9 \\
        &$\gamma_E = 0.50$ & 707.18$\pm$24.63  \\
        &$\gamma_E = 0.25$ & 691.33$\pm$13.60 \\
    Delayed QL  &  & 531.95$\pm$58.66    \\
    QL + $\epsilon$-greedy    &  & 2.98$\pm$0.012   \\
    \bottomrule
  \end{tabular}
\end{sc}
\end{small}
\end{center}
\vskip -0.1in
\end{table}

On all experiments, each algorithm ran for 10,000 timesteps and the undiscounted total sum of reward was recorded. Tables \ref{chainMDP-experiment_1} and \ref{chainMDP-experiment_2} show the average and 95\% confidence intervals over 300 independent test runs. The results show that Directed Delayed Q-learning significantly outperforms other model-free methods. Especially, we notice the gap between the performances of the algorithms increases exponentially as the chain length $N$ increases, which suggests that the larger value of $\gamma_E$ is beneficial especially in environments where reward is more sparse and deeper exploration is required. 

\section{Conclusion}
We presented Directed Delayed Q-learning, a provably efficient model-free reinforcement-learning algorithm which takes into account long term exploratory information. It has the same desirable sample complexity as Delayed Q-learning. The experiments show that Directed Delayed Q-learning shows significantly better performance compared to other model-free RL methods on challenging environments. %One future direction is to extend the results, using generalization, to richer world models with an infinite number of states and actions.

\newpage
\bibliography{example_paper}
\bibliographystyle{icml2018}

 \newpage
 \appendix

\onecolumn
\section{Analysis}

In this section, we show the proof of the main theoretical result, Theorem 1. The proofs follow the structure of the work of \cite{strehl2009reinforcement}, but specify some of steps for our proposed method. The following theorem (Theorem 10 in \citealt{strehl2009reinforcement}) will come in handy to show that our proposed algorithm is PAC-MDP. 
\vspace*{.3cm}

\begin{theorem}[\citealt{strehl2009reinforcement}]
Let $\mathfrak{A}(\epsilon, \delta)$ be any greedy learning algorithm such that, for every timestep $t$, there exists a set $K_t$ of state-action pairs that depends only on the agent’s history up to timestep $t$. We assume that $K_t = K_{t+1}$ unless, during timestep $t$, an update to some state-action value occurs or the escape event $A_K$ happens. Let $M_k$ be the known state-action MDP and $\pi_t$ be the current greedy policy, that is, for all states $s$, $\pi_t(s) = \arg\max_a Q'_t(s,a)$. Furthermore, assume $Q'_t(s,a) \leq V'_\textnormal{max}$ for all $t$ and $(s,a)$. Suppose that for any inputs $\epsilon$ and $\delta$, with probability at least $1-\delta$, the following conditions hold for all states $s$, actions $a$, and timesteps $t$: (a)  $V'_t(s) \geq V^*(s) - \epsilon$ (optimism), (b) $V'_t(s) - V^{\pi_t}_{{M_K}_t}(s) \leq \epsilon$ (accuracy), and (c) the total number of updates of action-value estimates plus the
number of times the escape event from $K_t$, $A_k$, can occur is bounded by $\zeta(\epsilon, \delta)$ (learning complexity). Then, when $\mathfrak{A}(\epsilon, \delta)$ is executed on any MDP $M$, it will follow a 4$\epsilon$-optimal policy from its current state on all but
\begin{equation*}
O \left(\frac{V'_\textnormal{max} \zeta(\epsilon,\delta)}{\epsilon(1-\gamma)} \ln\frac{1}{\delta} \ln\frac{1}{\epsilon(1-\gamma)} \right)
\end{equation*}
timesteps, with probability at least $1-2\delta$.
\end{theorem}

Recall that we define $Q'_t(s,a)$ to be $Q_t(s,a) + \frac{\lambda}{\sqrt{\log_{\eta} E_t(s,a)}}$ for convenience. We first bound the number of successful updates. Since every successful update of $Q'(s, a)$ results in a decrease of at least $\epsilon_1$ and $Q'(s, a)$ is initialized to $1/(1 - \gamma)$. We have at most 
$\kappa := \frac{1}{(1-\gamma)\epsilon}$ successful updates of a fixed state-action pair $(s, a)$. Therefore, the total number of successful updates is at most $SA\kappa$. So there can be at most $m(1+SA\kappa)$ attempted updates for each pair $(s,a)$. Hence, there are at most of $SAm(1+SA\kappa)$ total attempted updates.

Following the construction of the set of the low Bellman error state-action pairs in Delayed Q-learning \citep{strehl2009reinforcement}, during timestep $t$ of learning, we define $K_t$ to be the set of all state-action pairs $(s, a)$ such that: 
\begin{equation}\label{eqn:known_set}
Q'_t(s,a) - \left(R(s,a) + \gamma \sum_{s'} T(s'|s,a)V_t(s') \right) \leq 4\epsilon_1 \>.
\end{equation}

\begin{definition}Define \textbf{Event A1} to be the event that for all timesteps $t$, if $(s,a) \notin K_{t_1}$ and an attempted update of (s,a) occurs during timestep $t$, then the update will be successful, where $t_1 < t_2 < ... < t_m = t$ are $m$ last timesteps during which $(s,a)$ is experienced consecutively.
\end{definition}

During any given infinite-length execution of Directed Delayed Q-learning, when $(s,a) \notin K_{t_1}$ as above, our value function estimate $Q'(s, a)$ is very inconsistent with our current value function estimates. Thus,
we would expect our next attempted update to succeed. The next lemma shows that this update occurs with high probability. The proof of the lemma follows the structure of the lemma of \cite{strehl2009reinforcement}, but also bound the $E$-value estimates and specify additional parameter values.
We specify a value $m$ and first consider values $m_1 =\frac{(1+\gamma V_\textnormal{max})^2}{2\epsilon_1^2} \ln \left( \frac{6SA(1+SA\kappa)}{\delta} \right)$ and $m_2 =\frac{\gamma_E^2}{2\epsilon_1^2} \ln \left( \frac{6SA(1+SA\kappa)}{\delta} \right)$.

\begin{lemma}
The probability that \textbf{A1} is violated during execution of Directed Delayed Q-learning is at most $\delta/3$ with $m = m_1$, $\lambda \leq \epsilon\sqrt{\log_{\eta} \epsilon}$ and $\rho \leq \epsilon_1 \sqrt{m}$. 
\end{lemma}

\begin{proof}
Fix a state-action pair $(s,a)$ and suppose that it has been visited $m$ times until timestep $t$, at steps $t_1, ..., t_k$. Consider $m$ rewards, $r_{t_1},...,r_{t_m}$, and $m$ next states, $s_{t_1},...,s_{t_m}$ for $(s,a)$. Define the random variables $X_i := r_{t_i} + \gamma V_{t_1}(s_{t_i})$. Clearly, $0 \leq X_i \leq (1+\gamma V_\textnormal{max})$. Using the Hoeffding bound with the choice of $m_1$ above, it can be shown that
\begin{equation}\label{eqn:bound_on_X}
\frac{1}{m_1}\sum_{i=1}^{m_1} (r_{t_i}+\gamma V_{t_1}(s_{t_i}))-\mathbb{E}[X_1] < \epsilon_1
\end{equation}
holds with probability at least $1 - \delta/(6SA(1 + \kappa))$. Similarly, define the random variable $Y_i := \gamma_E G_{t_1}(s_{t_i})$ where $G(s) = \max_a E(s,a)$. Note that $0 \leq Y_i \leq \gamma_E$. Again, using the Hoeffding bound with the choice of $m_2$ above, it can be shown that 
\begin{equation}\label{eqn:bound_on_Y}
\frac{1}{m_2}\sum_{i=1}^{m_2} \gamma_E G_{t_1}(s_{t_i})-\mathbb{E}[Y_1] < \epsilon_1
\end{equation}
holds with probability at least $1 - \delta/(6SA(1 + \kappa))$. Note that given $\epsilon_1 < 0.5$, we can choose the constants $\eta \in (0,1)$ and $\lambda \leq \epsilon_1\sqrt{\log_{\eta} \epsilon_1}$ such that $E(s,a) \leq \epsilon_1$ implies $\frac{\lambda}{\sqrt{\log_{\eta} E(s,a)}} \leq \epsilon_1$. Here we choose $\eta = 1 - {\epsilon_1}$ (this choice will be useful when we bound $Q'$)

We choose $m \geq m_1 = \max(m_1, m_2)$ since $(1+\gamma V_\textnormal{max})^2 > \gamma_E^2$. Hence, it does not matter what value $\gamma_E$ is (as long as $\gamma_E \in (0,1)$) to determine the PAC-bound. We show that if an attempted update is performed for $(s,a)$ using these $m$ samples, then the resulting update will succeed with high probability. 

\begin{align*}
&Q'_t(s,a) - \left( \frac{1}{m}\sum_{i=1}^m (r_{t_i} +\gamma V_{t_i}(s_{t_i})) + \frac{\rho}{\sqrt{m}} + \frac{\lambda}{\sqrt{\log_{\eta} \frac{1}{m}\sum_{i=1}^m \gamma_E G_{t_i}(s_{t_i})}} \right)\\
&\geq Q'_t(s,a) - \left( \frac{1}{m}\sum_{i=1}^m (r_{t_i} +\gamma V_{t_1}(s_{t_i})) + \frac{\rho}{\sqrt{m}} + \frac{\lambda}{\sqrt{\log_{\eta} \frac{1}{m} \sum_{i=1}^m \gamma_E G_{t_1}(s_{t_i})}} \right)\\
&> Q'_t(s,a) - \mathbb{E}[X_1] - \epsilon_1 - \frac{\rho}{\sqrt{m}} - \frac{\lambda}{\sqrt{\log_{\eta} \mathbb{E}[Y_1]}} - \epsilon_1\\
&\geq 4\epsilon_1 - \frac{\rho}{\sqrt{m}} - 2\epsilon_1\\
&\geq 2\epsilon_1 - \frac{\rho}{\sqrt{m}} .
\end{align*}
The first inequality follows from $V_{t_i}(s) \leq V_{t_1}(s)$ and $G_{t_i}(s) \leq G_{t_1}(s)$ for all $s$ and $i$. The second inequality follows from \eqref{eqn:bound_on_X} and \eqref{eqn:bound_on_Y} along with a suitable choice of $\lambda$. The third step uses the assumption on \textbf{A1}, i.e. $(s,a) \notin K_{t_1}$, therefore \eqref{eqn:known_set} doesn't hold.
Hence, if we choose $\rho \leq \epsilon_1\sqrt{m}$, then with probability at least $1-3\delta$ we have:
\begin{equation}
Q'_t(s,a) - \left(\Tilde{Q}(s,a) + \frac{\rho}{\sqrt{n(s,a)}} + \frac{\lambda}{\sqrt{\log_{\eta} \Tilde{E}(s,a)}} \right) \geq \epsilon_1 \>.
\end{equation}
\end{proof}

The following lemma states that our proposed algorithm will maintain optimistic action values with high probability.

\begin{lemma}
During execution of Directed Delayed Q-learning, $Q_t'(s, a) \geq Q^*(s, a)$ holds for all timesteps $t$ and state-action pairs $(s, a)$, with probability at least $1 - \delta/3$.
\end{lemma}

\begin{proof}
Fix a state-action pair $(s,a)$ and suppose that it has been visited $k \leq m$ times until timestep $t$, at steps $t_1, ..., t_k$. Define the random variables $X_1,..., X_k$ by $X_i := r_{t_i} + \gamma V^*(s_{t_i+1}) $
Note that $\mathbb{E}[X_i] = Q^*(s,a)$ and $0 \leq X_i \leq 1 + \gamma V_\textnormal{max}$ for all $i = 1,...,k$ and the sequence $Q^*(s,a) - X_i$ is a martingale difference sequence. Applying Azuma's lemma, we have
\begin{equation}
P\left[ \mathbb{E}[X_1] - \frac{1}{k}\sum_{i=1}^k X_i \geq \frac{\rho}{\sqrt{k}} \right] \leq \exp \left( -\frac{\rho^2}{2(1+\gamma V_\textnormal{max})^2 } \right).
\end{equation}
Let the right-hand side be equal to $\frac{\delta}{3SAm(1+SA\kappa)}$. Then with $\rho \geq \left(1+\gamma V_\textnormal{max} \right) \sqrt{\frac{1}{2} \ln \frac{3SAm(1+SA\kappa)}{\delta}}$, we have that
\begin{equation}
\frac{1}{k}\sum_{i=1}^k (r_{t_i} + \gamma V^*(s_{t_i+1})) + \frac{\rho}{\sqrt{k}} \geq Q^*(s,a) 
\end{equation}
holds for all attempted updates, with probability at least $1-3\delta$. Assuming this equation does hold, the proof of the lemma is by induction on the timestep $t$. Note that since $Q'_t(s,a) \geq Q_t(s,a)$ for all $(s,a)$ and $t$, it suffices to show $Q_t(s,a) \geq Q^*(s,a)$ for all $t$. For the base case, note that $Q_1(s,a) = 1/(1-\gamma) \geq Q^*(s,a)$ for all $(s,a)$. Now, suppose that $Q_{t'}(s, a) \geq Q^*(s, a)$ holds true for all $t' \leq t$. Hence, $Q_t(s,a) \geq Q^*(s,a)$ and $V_t(s) \geq V^*(s)$  for all $(s,a)$. Then we have $Q_{t+1}(s,a) = \frac{1}{k}\sum_{i=1}^k (r_{t_i} + \gamma V_{t_i}(s_{t_i+1})) + \frac{\rho}{\sqrt{k}} \geq \frac{1}{k}\sum_{i=1}^k (r_{t_i} + \gamma V^*(s_{t_i+1})) + \frac{\rho}{\sqrt{k}} \geq Q^*(s,a) $.
\end{proof}

\begin{lemma}[\citealt{strehl2009reinforcement}]
The number of timesteps $t$ such that a state-action pair $(s,a) \notin K_t$ is experienced is at most $2mSA\kappa$.
\end{lemma}
\begin{proof}
See Lemma 25 in \cite{strehl2009reinforcement} for proof. Note that although the proposed algorithm can update Q-value and $E$-values estimates before $m$ attempts. it can take up to $m$ attempts (and still not succeed) in the worst case. Therefore, the analysis for this lemma is the same as \cite{strehl2009reinforcement}
\end{proof}

Next, we bound $Q'(s,a)$ for all state-action pair.

\begin{lemma}
If $\eta = 1-\epsilon_1$ and $\lambda \leq \epsilon_1\sqrt{\log_{\eta} \epsilon_1}$, then $Q'_t(s,a) \leq V_\textnormal{max} + \sqrt{1-\epsilon}$ for all $t$ and $(s,a)$.
\end{lemma}
\begin{proof}
Since $E_0(s,a) = 1 - \epsilon$, and $E_t(s,a) \leq E_0(s,a)$, we have
\begin{align*}
Q'_t(s,a) &= Q_t(s,a) + \frac{\lambda}{\sqrt{\log_{\eta} E_t(s,a)}}\\
&\leq V_\textnormal{max} + \epsilon_1\sqrt{\log_{\eta} \epsilon_1} 
= V_\textnormal{max} + \epsilon_1\sqrt{\frac{\log \epsilon_1}{\log (1-\epsilon_1) }}\\
&\leq V_\textnormal{max} + \epsilon_1\sqrt{\frac{\frac{1}{\epsilon_1}-1}{\epsilon_1 }}
= V_\textnormal{max} + \sqrt{1-\epsilon}
\end{align*}
\end{proof}

Using these Lemmas we can prove the main result, Theorem 1.

\begin{proof} (of Theorem 1)
We show that combining Lemmas satisfies the conditions of Theorem 2. First, set $m$ as in Lemma 1 and let $\epsilon_1 = \epsilon(1-\gamma)/4, \eta = 1-\epsilon_1, \rho = \epsilon_1\sqrt{m}$ and $\lambda = \epsilon_1\sqrt{\log_{\eta} \epsilon_1} $. Let $V'_\textnormal{max} = V_\textnormal{max} + \sqrt{1-\epsilon}$. Then, by Lemma 4, $Q'_t(s,a) \leq V'_\textnormal{max}$ for all $t$ and $(s,a)$. By Lemma 1, event \textbf{A1} holds with probablity at least $1-\delta/3$. Then, the optimism condition (a) $V'_t(s) \geq V^*(s) - \epsilon$ is satisfied by Lemma 2. Note that for all $(s,a)$ if $(s,a) \in K_t$, then equation (1) holds. Otherwise $Q'_t(s,a) = Q^{\pi_t}_{M_{K_t}}(s,a)$. Hence, $V'_t(s,a)$ and $V^{\pi_t}_{M_{K_t}}(s,a)$ can be off by at most $4\epsilon$ in reward at each time $t$. Therefore, $V'_t(s,a) - V^{\pi_t}_{M_{K_t}}(s,a) \leq \frac{4\epsilon_1}{1-\gamma} = \epsilon$, which satisfies condition (b); see, e.g. \cite{strehl2009reinforcement}. Now, from Lemma 3, we have $\zeta(\epsilon, \delta) = 2mSA\kappa$, where $\zeta(\epsilon, \delta)$ is the number of updates and escape events that occur during execution of Directed Delayed Q-learning. Hence, putting the results together, the algorithm will follow a $\epsilon$-optimal policy from its current state on all but
\begin{equation*}
O\left( \frac{SA}{\epsilon^4(1-\gamma)^8}\ln\frac{1}{\delta}\ln\frac{1}{\epsilon(1-\gamma)}\ln\frac{SA}{\delta\epsilon(1-\gamma)} \right)
\end{equation*}
timesteps.
\end{proof}

\end{document}